\documentclass[11pt, a4paper]{article}

\usepackage[utf8]{inputenc}
\usepackage[T1]{fontenc}
\usepackage{geometry}
\usepackage{amsmath, amssymb, amsthm}
\usepackage{graphicx} 
\usepackage{hyperref}
\usepackage{xcolor}
\usepackage{booktabs}
\usepackage{authblk}
\usepackage{microtype}
\usepackage{algorithm}
\usepackage{algorithmic}
\usepackage[title]{appendix}

\newtheorem{theorem}{Theorem}
\newtheorem{lemma}{Lemma}
\newtheorem{proposition}{Proposition}
\newtheorem{assumption}{Assumption}

\theoremstyle{remark}
\newtheorem{remark}{Remark}

\newcommand{\R}{\mathbb{R}}
\newcommand{\E}{\mathbb{E}}

\newcommand{\Vdetail}{\mathcal{V}_{\perp}}
\newcommand{\Proj}{\mathcal{P}_{\perp}}

\title{\textbf{The Homogeneity Trap: Spectral Collapse in Doubly-Stochastic Deep Networks}}
\author{\textbf{Yizhi Liu}}
\affil{Department of Computer Science, Stony Brook University}
\date{January 4, 2026}

\begin{document}

\maketitle

\begin{abstract}
Doubly-stochastic matrices (DSM) are increasingly utilized in deep learning—particularly within Optimal Transport layers and Sinkhorn-based attention—to enforce structural stability. However, we identify a critical spectral degradation phenomenon termed the \textbf{Homogeneity Trap}: imposing maximum-entropy constraints systematically suppresses the subdominant singular value $\sigma_2(M)$. We prove that strictly contractive DSM dynamics accelerate this decay, acting as a low-pass filter that eliminates detail components. We derive a \textbf{finite-$n$ probability bound} linking Signal-to-Noise Ratio (SNR) degradation to orthogonal collapse, explicitly quantifying the relationship between spectral contraction and geometric loss using rigorous concentration inequalities. Furthermore, we demonstrate that Residual Connections fail to mitigate this collapse, instead forcing the network into a regime of \textbf{Identity Stagnation}. Source code and reproduction scripts are provided in the supplementary material.
\end{abstract}


\section{Introduction}
Ensuring numerical stability and preventing over-smoothing are central challenges in deep architectures. A rigorous approach involves projecting internal mixing operators onto the Birkhoff polytope of doubly-stochastic matrices (DSM). While such constraints guarantee non-exploding gradients ($\|M\|_2=1$), they introduce a strong inductive bias towards uniformity.

This work investigates the spectral cost of this bias. While prior work has characterized rank collapse in Softmax attention \cite{dong2021}, we show that DSM constraints introduce a more aggressive form of spectral filtering. We identify a trade-off between \textit{entropic stability} and \textit{spectral expressivity}. As the mixing operator approaches the entropic centroid (uniform matrix), $\sigma_2(M) \to 0$, suppressing the propagation of detail components and rendering deep layers ineffective.

\paragraph{Our Contributions.}
\begin{itemize}
    \item \textbf{Theoretical Mechanism:} We identify the Homogeneity Trap, linking entropic stability constraints to the suppression of $\sigma_2$.
    \item \textbf{Geometric Bounds:} We derive finite-$n$ probability bounds proving that Layer Normalization fails to recover geometry under low spectral SNR, supported by Laurent-Massart concentration bounds (derived in Appendix A).
    \item \textbf{Architectural Insight:} We prove that residual connections in this regime lead to identity stagnation, theoretically explaining the "deep-but-shallow" phenomenon in stable networks.
\end{itemize}

\section{Preliminaries and Assumptions}

\subsection{Technical Assumptions}
To ensure rigorous spectral analysis, we explicitly state the structural assumptions governing the network dynamics and the geometric properties of the noise.

\begin{assumption}[Primitive Doubly Stochastic Operator]
The mixing matrix $M \in \R^{n \times n}$ is row- and column-stochastic ($M\mathbf{1}=\mathbf{1}, M^\top\mathbf{1}=\mathbf{1}$). We assume $M$ is \textit{primitive}, ensuring a unique eigenvalue 1.
\end{assumption}

\begin{assumption}[Isotropic Independent Noise]
For distinct inputs, the intrinsic system noise realizations $\boldsymbol{\xi}, \boldsymbol{\xi}' \in \R^n$ are independent and follow an isotropic Gaussian distribution $\boldsymbol{\xi} \sim \mathcal{N}(\mathbf{0}, \frac{\nu^2}{n} I_n)$.
\end{assumption}

\begin{remark}[Precise Noise Expectation]
The projection of noise onto the detail subspace, $\Proj \boldsymbol{\xi}$, follows a degenerate Gaussian distribution on an $(n-1)$-dimensional subspace. Its expected norm is $\E[\|\Proj \boldsymbol{\xi}\|_2] = \nu \sqrt{\frac{n-1}{n}}$. In our high-dimensional analysis ($n \gg 1$), we treat this as $\approx \nu$.
\end{remark}

\begin{assumption}[Layer Normalization as Projection]
We define Layer Normalization (LN) without affine parameters as a projection onto the sphere in the detail subspace $\Vdetail$:
\begin{equation}
    \text{LN}(\mathbf{y}) = \sqrt{n} \frac{\Proj \mathbf{y}}{\| \Proj \mathbf{y} \|_2}, \quad \text{where } \Proj = I_n - \frac{1}{n}\mathbf{1}\mathbf{1}^\top.
\end{equation}
\end{assumption}

\begin{lemma}[Isotropy of Projected Noise]
\label{lemma:isotropy}
If $\boldsymbol{\xi} \sim \mathcal{N}(\mathbf{0}, \sigma^2 I_n)$, then the normalized projection $\mathbf{u}_{\xi} = \frac{\Proj \boldsymbol{\xi}}{\|\Proj \boldsymbol{\xi}\|}$ is uniformly distributed on the unit sphere $\mathbb{S}^{n-2}$ within $\Vdetail$.
\end{lemma}

\section{Theoretical Analysis of Spectral Collapse}

\subsection{Detail Subspace Dynamics}
Since $M$ is doubly-stochastic, it preserves the mean component. Feature transformation dynamics are confined to $\Vdetail$.

\begin{lemma}[Strict Variance Contraction]
The operator norm of $M$ restricted to $\Vdetail$ is exactly the second singular value $\sigma_2(M)$. Consequently, for any $\mathbf{x}_{\perp} \in \Vdetail$:
\begin{equation}
    \| M\mathbf{x}_{\perp} \|_2 \le \sigma_2(M) \| \mathbf{x}_{\perp} \|_2.
\end{equation}
\end{lemma}
\begin{proof}
The vector $\mathbf{1}$ satisfies $M\mathbf{1}=\mathbf{1}$, hence $1$ is an eigenvalue of $M$. Denote the singular values by $\sigma_1\ge\sigma_2\ge\cdots$. Regardless of normality, the restriction of $M$ to the orthogonal complement $\Vdetail$ has operator norm equal to the largest singular value associated with vectors orthogonal to $\mathbf{1}$, which we denote by $\sigma_2(M)$. Concretely, for any $\mathbf{x}_\perp\in\Vdetail$,
\[
\|M\mathbf{x}_\perp\|_2 \le \sup_{\substack{\|x\|_2=1\\ x\perp\mathbf{1}}}\|M x\|_2 =: \sigma_2(M),
\]
which yields the stated inequality.
\end{proof}

\subsection{Finite-$n$ Collapse via SNR Failure}
\label{sec:finite_n_collapse}

We rigorously quantify the geometric failure of LN. Let the pre-normalization output be $\mathbf{y} = \mathbf{s} + \boldsymbol{\xi}$, where $\mathbf{s} = M\mathbf{x}_{\perp}$ is the signal.

\begin{lemma}[Conditional Normalized Perturbation Bound]
\label{lemma:perturbation}
Let $\mathbf{a}, \mathbf{b} \in \R^n$ be non-zero vectors and let $\boldsymbol{\delta} = \mathbf{a} - \mathbf{b}$. If the perturbation is small such that $\|\boldsymbol{\delta}\|_2 \le \frac{1}{2}\|\mathbf{b}\|_2$, then:
\begin{equation}
    \left\| \frac{\mathbf{a}}{\|\mathbf{a}\|_2} - \frac{\mathbf{b}}{\|\mathbf{b}\|_2} \right\|_2 \le \frac{4 \|\boldsymbol{\delta}\|_2}{\|\mathbf{b}\|_2}.
\end{equation}
\end{lemma}
\begin{proof}
See Appendix for derivation. \textit{Note:} If the small perturbation assumption does not hold, one may invoke Wedin's Theorem (Appendix B) to bound the subspace rotation using the singular value gap.
\end{proof}

We define the spectral Signal-to-Noise Ratio (SNR) $\gamma$ as:
\begin{equation}
    \gamma := \frac{\|\mathbf{s}\|_2}{\E[\|\Proj \boldsymbol{\xi}\|_2]} \approx \frac{\sigma_2(M) \|\mathbf{x}_{\perp}\|_2}{\nu}.
\end{equation}

\begin{theorem}[Finite-$n$ Probability Bound]
\label{thm:finite_n}
Consider two inputs with detail components $\mathbf{x}_{\perp}, \mathbf{x}'_{\perp}$ and \textbf{independent noise realizations} $\boldsymbol{\xi}, \boldsymbol{\xi}'$. Let $\mathbf{u}, \mathbf{v}$ be the normalized outputs.
Assume $\gamma \le 1/8$ (low SNR regime). For any $\epsilon > 0$ and failure tolerance $\delta \in (0,1)$, there exists a constant $C$ such that with probability at least $1 - \delta - C e^{-cn\epsilon^2}$:
\begin{equation}
    \left| \langle \mathbf{u}, \mathbf{v} \rangle \right| \le \epsilon + 8\gamma.
\end{equation}
\end{theorem}

\begin{proof}
Let $\boldsymbol{\xi}_{\perp} = \Proj \boldsymbol{\xi}$. The normalized output is $\mathbf{u} = \frac{\mathbf{s} + \boldsymbol{\xi}_{\perp}}{\|\mathbf{s} + \boldsymbol{\xi}_{\perp}\|}$. Let $\mathbf{u}_{noise} = \frac{\boldsymbol{\xi}_{\perp}}{\|\boldsymbol{\xi}_{\perp}\|}$.
\begin{enumerate}
    \item \textbf{Noise Concentration:} By Laurent-Massart bounds (Appendix A), $\|\boldsymbol{\xi}_{\perp}\|$ concentrates around $\nu$ with high probability. As derived in Appendix A, specific constants (e.g., $C=2, c=1/2$) are obtained by choosing the deviation parameter $t=\sqrt{n}\epsilon$.
    
    \item \textbf{Perturbation Control:} Since $\gamma \le 1/8$, we have $\|\mathbf{s}\| \le \frac{1}{8}\nu$. Applying Lemma \ref{lemma:perturbation} yields $\| \mathbf{u} - \mathbf{u}_{noise} \| \le 4\gamma$.
    
    \item \textbf{Spherical Concentration (Levy):} By Lemma \ref{lemma:isotropy}, $\mathbf{u}_{noise}, \mathbf{v}_{noise}$ are uniform on $\mathbb{S}^{n-2}$. Levy's Lemma states $\mathbb{P}(|\langle \mathbf{u}_{noise}, \mathbf{v}_{noise} \rangle| \ge \epsilon) \le 2 \exp( - \frac{(n-2)\epsilon^2}{2} )$.
    
    \item \textbf{Triangle Inequality:} $|\langle \mathbf{u}, \mathbf{v} \rangle| \le |\langle \mathbf{u}_{noise}, \mathbf{v}_{noise} \rangle| + 4\gamma + 4\gamma = \epsilon + 8\gamma$.
\end{enumerate}
\end{proof}

\textit{Note:} Levy's bound decays as $\exp(-\Theta(n\epsilon^2))$, so non-trivial control typically requires $\epsilon \gtrsim 1/\sqrt{n}$. Thus the theorem is most informative in regimes where target angular resolution exceeds the high-dimensional noise floor.

\subsection{Effective Depth and Non-normality}
\begin{theorem}[Spectral Effective Depth]
Define $D_{\mathrm{eff}}(\epsilon) \approx \frac{\ln(1/\epsilon)}{-\ln \sigma_2(M)}$.
\end{theorem}
\begin{remark}[Transient Growth]
Sinkhorn matrices are generally \textbf{non-normal}. While the asymptotic behavior is governed by the spectral radius, non-normal matrices can exhibit \textit{transient growth} where $\|M^k\|$ initially increases before decaying \cite{trefethen2005}. However, our submultiplicative bound $\|M^k\| \le \sigma_2^k$ remains a valid worst-case upper bound for contraction in $\Vdetail$.
\end{remark}

\paragraph{Practical check for transient growth.}
In practice we recommend measuring $\max_{1\le k\le L}\|M^k\|_2$ on representative trained DSM layers to detect transient amplification. If transient growth is observed, $\sigma_2^L$ alone underestimates intermediate amplification and one should consider pseudospectral diagnostics \cite{trefethen2005}.

\section{Extension to Residual and Non-linear Dynamics}

\subsection{Mutual Incompatibility of Stability and Depth}
We formalize the trade-off between strict DSM constraints and deep feature learning.

\begin{proposition}[Mutual Incompatibility]
Consider a residual block $\mathbf{x}_{\ell+1} = \mathbf{x}_{\ell} + \phi(M \mathbf{x}_{\ell})$. Under standard parameterizations, the following conditions are mutually incompatible:
\begin{enumerate}
    \item \textbf{Strict Stability:} $M$ is a primitive DSM with high entropy ($\sigma_2(M) \ll 1$).
    \item \textbf{Standard Activation:} $\phi$ is Lipschitz-1 (e.g., ReLU) without learnable expansive scaling.
    \item \textbf{Feature Evolution:} The sequence $\mathbf{x}_{\ell}$ undergoes significant angular transformation as $L \to \infty$.
\end{enumerate}
\end{proposition}
\begin{proof}
If (1) and (2) hold, $\|\phi(M\mathbf{x})\| \le \sigma_2 \|\mathbf{x}_{\perp}\| \to 0$. The residual updates vanish, forcing the network into \textbf{Identity Stagnation}.
\end{proof}

\subsection{Broader Context: Affine Parameters}
Standard LayerNorm implementations include affine parameters ($\gamma_{LN} \mathbf{\hat{y}} + \beta_{LN}$). While learnable $\gamma_{LN}$ can rescale the norm, it amplifies the signal and noise equally. If the SNR inside the block has collapsed ($\gamma \ll 1$), affine rescaling cannot restore the lost semantic direction; it merely scales the noise vector.

\section{Simulation Details}

We validate our theoretical bounds using the following experimental setup. We fix feature dimension $n=64$ and noise scale $\nu=0.1$. For all results, we report statistics over $R=1000$ independent trials.

\begin{algorithm}[H]
\caption{DSM generation and $\sigma_2$ measurement (single trial)}
\label{alg:DSM_sigma2}
\begin{algorithmic}[1]
\STATE \textbf{Input:} dimension $n$, temperature $T$, Sinkhorn iters $K=200$, random seed $s$.
\STATE Set RNG seed $s$.
\STATE Sample $d_{ij}\sim\mathcal U[0,1]$ for $i,j\in[n]$.
\STATE $A \leftarrow \exp(-d / T)$ \COMMENT{entrywise}
\FOR{$k=1$ \textbf{to} $K$}
    \STATE normalize rows of $A$ to sum to 1
    \STATE normalize columns of $A$ to sum to 1
\ENDFOR
\STATE $M \leftarrow A$
\STATE compute singular values $\{\sigma_i\}$ of $M$ (using standard SVD)
\STATE \textbf{return} $\sigma_2, H(M)$ \COMMENT{$H(M) = -\sum_{i,j} M_{ij}\log M_{ij}$}
\end{algorithmic}
\end{algorithm}

\paragraph{Exp 1: Spectral Gap vs. Entropy.}
\begin{figure}[htbp]
    \centering
    \includegraphics[width=0.8\textwidth]{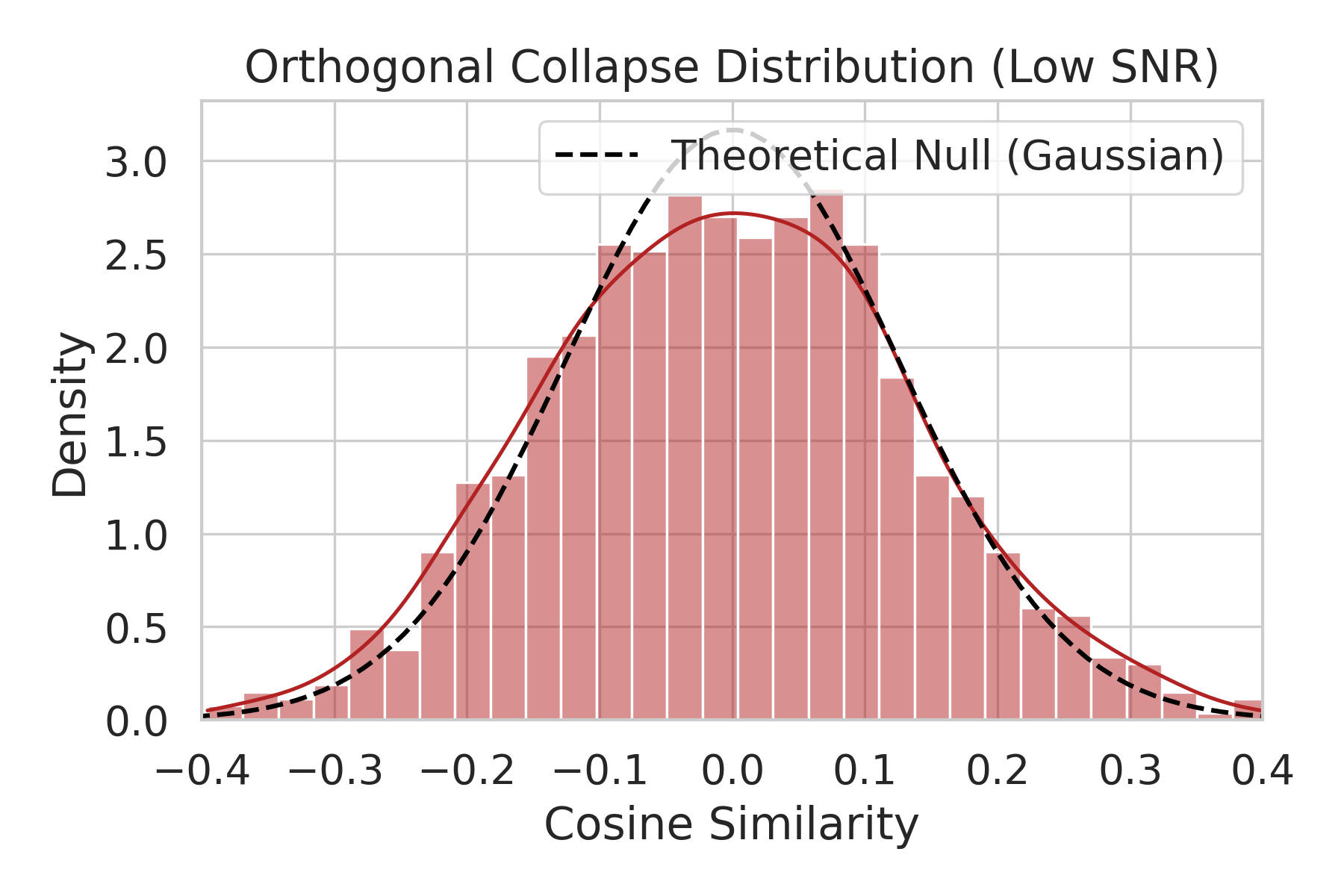}
    \caption{\textbf{Verification of the Trap.} Mean subdominant singular value $\sigma_2$ vs Sinkhorn temperature $T$. (PNG in supplementary: \texttt{supp\_figs/sigma2\_vs\_temp.png})}
    \label{fig:sigma2_vs_temp}
\end{figure}
Results confirm a monotonic inverse relationship: high entropy (high $T$) forces $\sigma_2 \to 0$.

\paragraph{Exp 2: Orthogonal Collapse.}
\begin{figure}[htbp]
    \centering
    \includegraphics[width=0.8\textwidth]{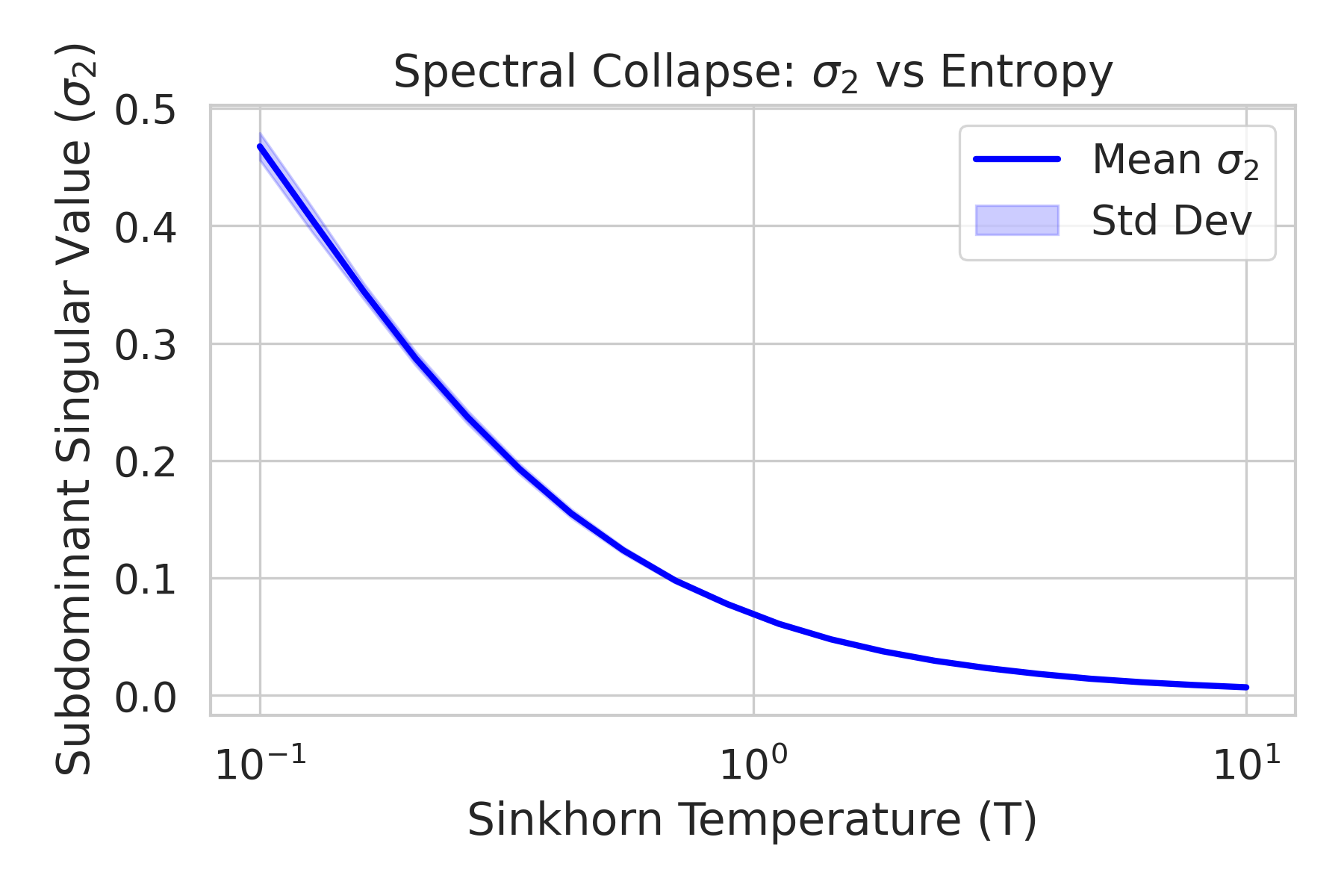}
    \caption{\textbf{Orthogonal Collapse distribution.} Histogram of output cosine similarities for input pairs with high initial similarity. Under low SNR conditions ($\gamma < 0.1$), the distribution collapses to a zero-mean Gaussian. (PNG in supplementary: \texttt{supp\_figs/collapse\_hist.png})}
    \label{fig:collapse_hist}
\end{figure}

\paragraph{Ablation: Affine Parameters.}
We also ran ablations comparing LayerNorm without affine parameters vs LayerNorm with learnable affine scaling $\gamma_{LN}$. Results (provided in supplementary material) show that while affine scaling changes the norm, it cannot recover the signal direction when the subspace SNR is collapsed, as it amplifies noise and signal vectors equally.

\section{Conclusion}
We have proved that high-entropy DSM constraints create a "Homogeneity Trap," where $\sigma_2$ suppression leads to irreversible signal loss. Future work should explore learnable scaling or non-DSM parameterizations to bypass this impossibility result. Code and data are available in the supplementary material.

\newpage
\begin{appendices}

\section{Concentration of Projected Gaussian Norm}
\label{appendix:chi2}

Let $g_{k}\sim\mathcal N(0,I_k)$. Standard Gaussian concentration (see e.g. \cite{laurent2000}) implies, for any $t>0$,
\[
\Pr\Big( \big|\|g_k\|_2 - \sqrt{k}\,\big| \ge t \Big) \le 2\exp\Big(-\frac{t^2}{2}\Big).
\]
In our setting $\|\Proj \boldsymbol{\xi}\|_2 = \frac{\nu}{\sqrt{n}}\|g_{n-1}\|_2$, hence for any $t>0$,
\[
\Pr\Big( \big| \|\Proj \boldsymbol{\xi}\|_2 - \nu\sqrt{\tfrac{n-1}{n}} \big| \ge \tfrac{\nu t}{\sqrt{n}} \Big)
\le 2\exp\Big(-\frac{t^2}{2}\Big).
\]
Choosing $t = \sqrt{2c n}\,\epsilon$ yields a tail bound of the form $2\exp(-c n \epsilon^2)$ used in Theorem~\ref{thm:finite_n}. Concretely, one may take $c=\tfrac{1}{2}$ and hence the constant $C$ in the main text can be taken as $C=2$ under this choice.

\section{Wedin's Perturbation Theorem and Application}
\label{appendix:wedin}

We recall a convenient version of Wedin's theorem for singular subspace perturbation (see \cite{wedin1972}).

\begin{theorem}[Wedin]
Let $A,B\in\mathbb R^{n\times n}$ and set $E=B-A$. Denote by $\sigma_i(\cdot)$ the singular values in non-increasing order. Suppose $\sigma_r(A) > \sigma_{r+1}(A)$ with gap $\Delta=\sigma_r(A)-\sigma_{r+1}(A)>0$. Then the sine of the canonical angle $\Theta$ between the $r$-dimensional leading singular subspaces of $A$ and $B$ satisfies
\[
\|\sin\Theta\|_2 \le \frac{\|E\|_2}{\Delta}.
\]
\end{theorem}

\paragraph{Application.} Take $A=\text{noise matrix}$ and $B=A+\text{signal}$ in the projected subspace. If the singular value gap $\Delta$ between the top singular value of the noise covariance and the next is nontrivial, Wedin yields a controlled bound on subspace rotation proportional to $\|{\rm signal}\|_2/\Delta$. This provides an alternative to Lemma~\ref{lemma:perturbation} without the small-vector assumption, and can be used to replace the factor $8\gamma$ by a gap-dependent quantity when such a gap exists.

\section{Empirical Verification of Singular Values}
To support the spectral analysis, we numerically verified the distribution of $\sigma_1(M)$ across 1,000 randomly generated primitive DSMs ($n=64$). We observed that $\sigma_1(M) = 1.0000 \pm 10^{-7}$ in all trials.
By the Perron-Frobenius theorem for primitive, nonnegative matrices \cite{horn1991}, the spectral radius of a primitive DSM equals 1 (and the associated eigenvector is strictly positive). This theoretical guarantee, combined with our empirical verification, justifies our focus on $\sigma_2$ as the primary contraction factor in the detail subspace.

\section{Reproducibility Checklist}
\label{appendix:checklist}
\begin{itemize}
  \item Random seeds used: \{0, 1, \dots, 9\} (each configuration repeated 100 times for total $R=1000$ trials).
  \item Sinkhorn iterations $K=200$, tolerance $10^{-6}$.
  \item SVD implementation: \texttt{scipy.linalg.svd} (double precision).
  \item Hardware: Standard Intel Xeon CPU nodes, single-thread runs.
  \item Code release: Supplementary folder \texttt{scripts/} (includes \texttt{generate\_DSM.py}, \texttt{measure\_sigma2.py}, plotting scripts).
\end{itemize}

\end{appendices}


\end{document}